\documentclass{article} 
\usepackage{times}  
\usepackage{helvet} 
\usepackage{courier}  
\usepackage[hyphens]{url}  
\usepackage{graphicx} 
\usepackage{natbib}  
\usepackage{caption} 
 \usepackage[linesnumbered,noend,ruled]{algorithm2e}
 \usepackage{amssymb,amstext,amsmath, amsthm}
\usepackage{mathtools}
\usepackage{subcaption}
\usepackage{xcolor}
\usepackage{algorithmic}

\newcommand{\<}{\langle}
\renewcommand{\>}{\rangle}

\newcommand{\bs}{\boldsymbol}
\newcommand{\bm}{\mathbf}

\newcommand{\cD}{\mathcal{D}}

\newcommand{\cN}{\mathcal{N}}

\newcommand{\cP}{\mathcal{P}}

\newcommand{\bE}{\mathbb{E}}

\newcommand{\bI}{\mathbf{I}}

\newcommand{\one}{\mathbf{1}}

\newcommand{\mCov}{\mathrm{Cov}}

\newcommand{\ivar}{{IVAR}}
\newcommand{\icov}{{ICOV}}

\newtheorem{assumption}{Assumption}

\newtheorem{proposition}{Proposition}

\newtheorem{example}{Example}
\newtheorem{remark}{Remark}

\SetKwInput{KwInput}{Input}                
\SetKwInput{KwOutput}{Output}              

\title{Probabilistic Inference for Learning from Untrusted Sources}
\author{
	Duc Thien Nguyen \and
	Shiau Hong Lim\and
	Laura Wynter\and
	Desmond Cai
	\thanks{The authors are with IBM Research, Singapore. Emails:
		\{Duc.Thien.Nguyen@, shonglim@sg., lwynter@sg, desmond.cai1@.\}ibm.com}
}

\begin{document}

\maketitle

\begin{abstract}

Federated learning  brings  potential   benefits of faster learning,  better solutions, and a greater propensity to transfer when heterogeneous data from different parties increases  diversity. However, because federated learning  tasks  tend to be large and complex, and training times non-negligible, it is important for the aggregation algorithm to be robust  to non-IID data and  corrupted parties. This robustness relies on the ability to identify, and appropriately weight, incompatible parties. Recent work assumes that a \textit{reference dataset} is available through which to perform the identification. We consider  settings where no such reference dataset is available; rather, the quality and suitability of the parties needs to be \textit{inferred}. We do so by bringing ideas from
crowdsourced predictions and
collaborative filtering, where one must infer an unknown ground truth
given proposals from participants with unknown quality.
We propose novel federated learning aggregation algorithms based on Bayesian inference
that adapt to the quality of the parties. Empirically,
we show that the
 algorithms outperform standard and robust aggregation
 in federated learning on both synthetic and
real data.
\end{abstract}

\section{Introduction}

For deep neural networks to address more complex tasks in the future it is likely that the participation of multiple users, and hence multiple sources of data, will need  become more widespread. This practice has been widely used  in object recognition \cite{crowdfaces, distribimage, edge, distrib2, distrib3}, but less so in domains such as finance, medicine, prediction markets, internet of things, etc.  Federated learning, as defined by \cite{FL0} is an answer to the problem of  training   complex, heterogeneous tasks. It involves  
 distributing model training  across a  number of parties in a centralized manner while taking into account  communication requirements, over potentially remote or mobile devices, privacy concerns  requiring  that data remains at the remote location, and the lack of balanced or IID data across parties.
 
One challenge in federated learning, as noted by \citet{FL3}, is   the quality and data distribution of the sources being used for the training tasks. A related challenge is the 
potential for random failures or  adversarial parties to disrupt the federated training.
For these reasons, \textit{robust} federated learning has  seen a flurry of activity \citet{FL5, FL6, FL7, FL9, FL10}.
Some, like \citet{byz1, FL6,   zeno} focus on the adversarial setting, and others, like \citet{untrusted, FL10}, focus on the  general setting of distributed learning under different source distributions. 
In both cases, this requires identifying the weight with which to include each party in the aggregation. 

 \citet{untrusted} proposed to give  the aggregator a \textit{reference dataset} with which to measure the quality of each party update.
Like \citet{untrusted}, we explore the question of efficient federated learning  with unequal and possibly untrusted parties. However, the assumption of access to a reference dataset is, for many real-world problems, problematic. Consider a federation of medical diagnosis facilities, each with its own patient population. Not only would it violate privacy concerns  to generate a reference dataset but it  would not in fact be feasible. The same problem arises in virtually any real-world domain for which federated learning offers an appealing solution.

We propose instead to  adapt inference methods from collaborative filtering \cite{CF} to the problem of heterogeneous federated learning aggregation. Using a Gaussian model, we model  each party's estimate  as a noisy observation of an unknown ground truth and define new  probabilistic inference algorithms to iteratively estimate the ground truth.  We show that   the estimated  ground truth   is robust to  faulty and poor quality data. 
Specifically, the contributions of this work are  as follows:
\begin{itemize}

\item We provide a maximum likelihood estimator of the uncertainty level of each party in a federated learning training task. The estimator gives rise to an appropriate weighting for each party in each aggregation.  When each party's data sample is independent,  the estimator reduces to the standard averaging scheme of \cite{McMahanMRHA17}; in the more general case of overlapping samples, it offers a new maximum likelihood estimator.
\item We define two new algorithms for federated learning that make use of the MLE: an inverse variance weighting  and an inverse covariance weighting scheme.
\item As the maximum likelihood estimator can overfit when the  available data is scarce and tends to be  computationally expensive for the inverse covariance scheme, we define a new Variational Bayesian (VB) approach to approximate the posterior distributions of the ground truth under both  independent and latent noise models.

\end{itemize}

Both the MLE and VB methods are tested  on synthetic and real datasets; the tests show the superiority of   aggregation  with  probabilistic inference over standard baselines including the mean  and the more robust median-based approaches: geometric median and  coordinate-wise median.


\begin{figure*}[htb]
	\centering
	\begin{subfigure}{0.25\textwidth}
		\includegraphics[width=\textwidth]{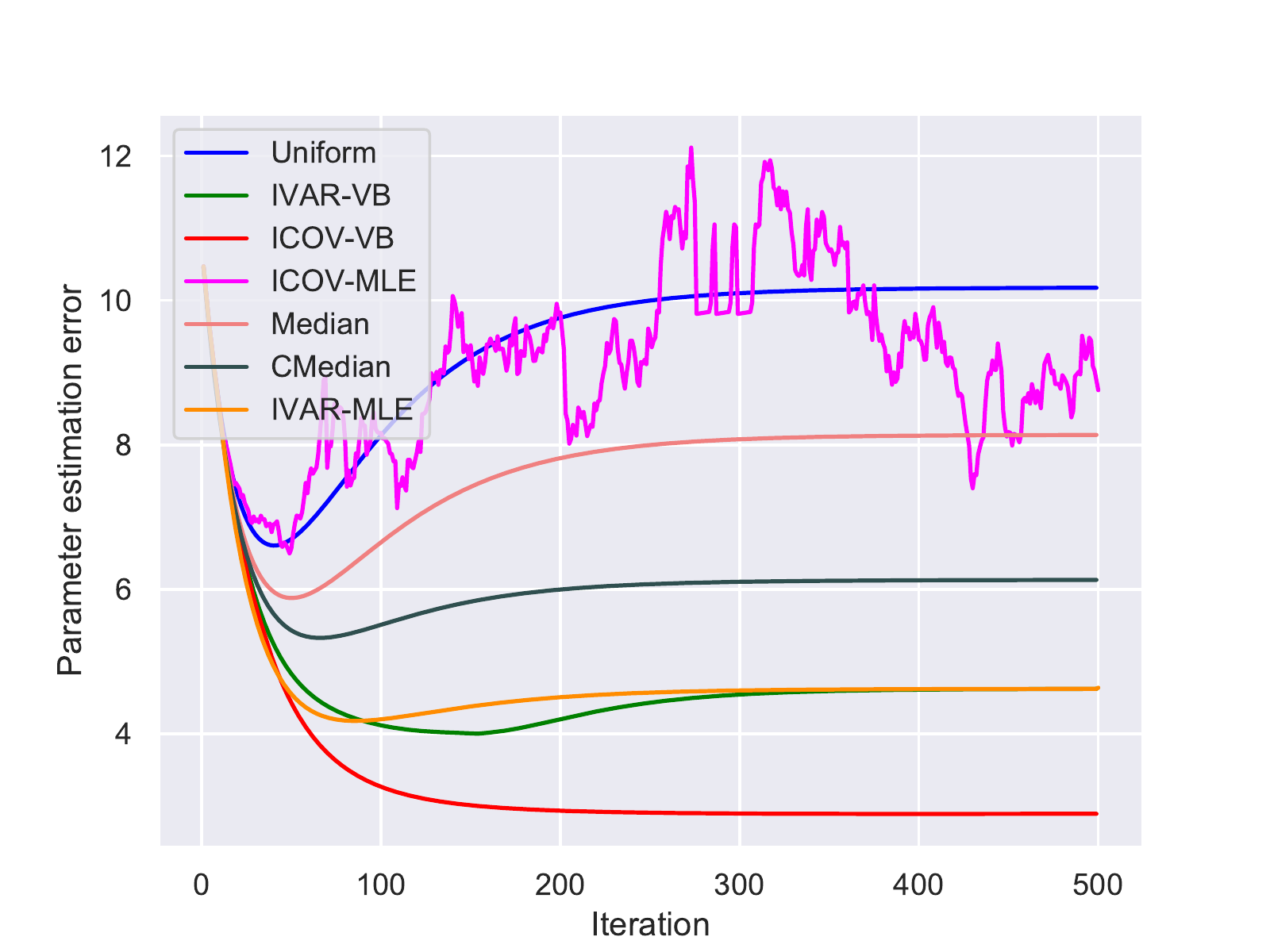} \caption{Full participation. Full batch   (300  samples)}\label{fig:fullP_fullB}
	\end{subfigure} \hskip 0.2cm
	\begin{subfigure}{0.25\textwidth}
		\includegraphics[width=\textwidth]{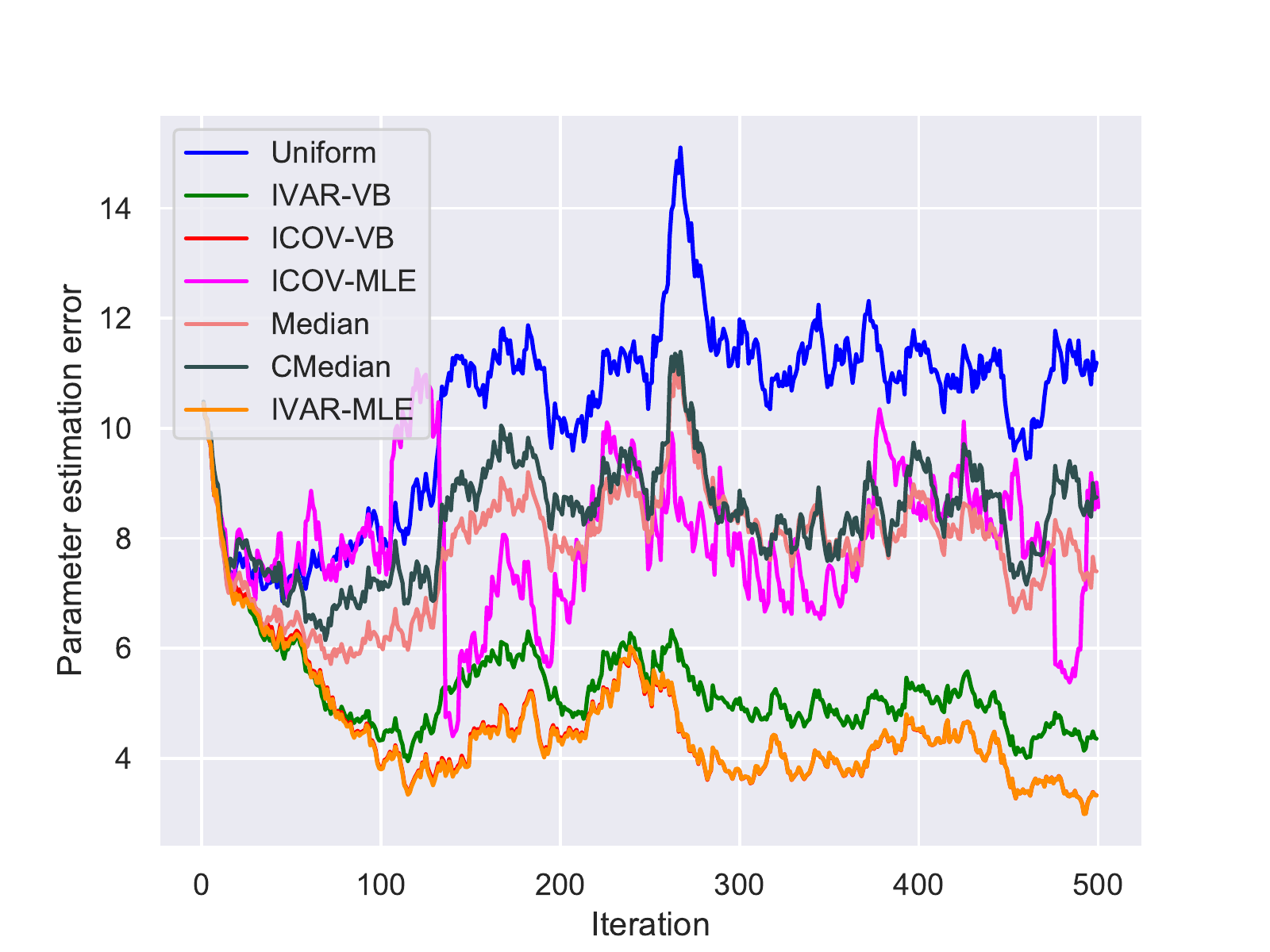} \caption{Full participation. Mini batch  (32 samples)}\label{fig:fullP_mB}
	\end{subfigure} \hskip 0.2cm
	\begin{subfigure}{0.25\textwidth}
		\includegraphics[width=\textwidth]{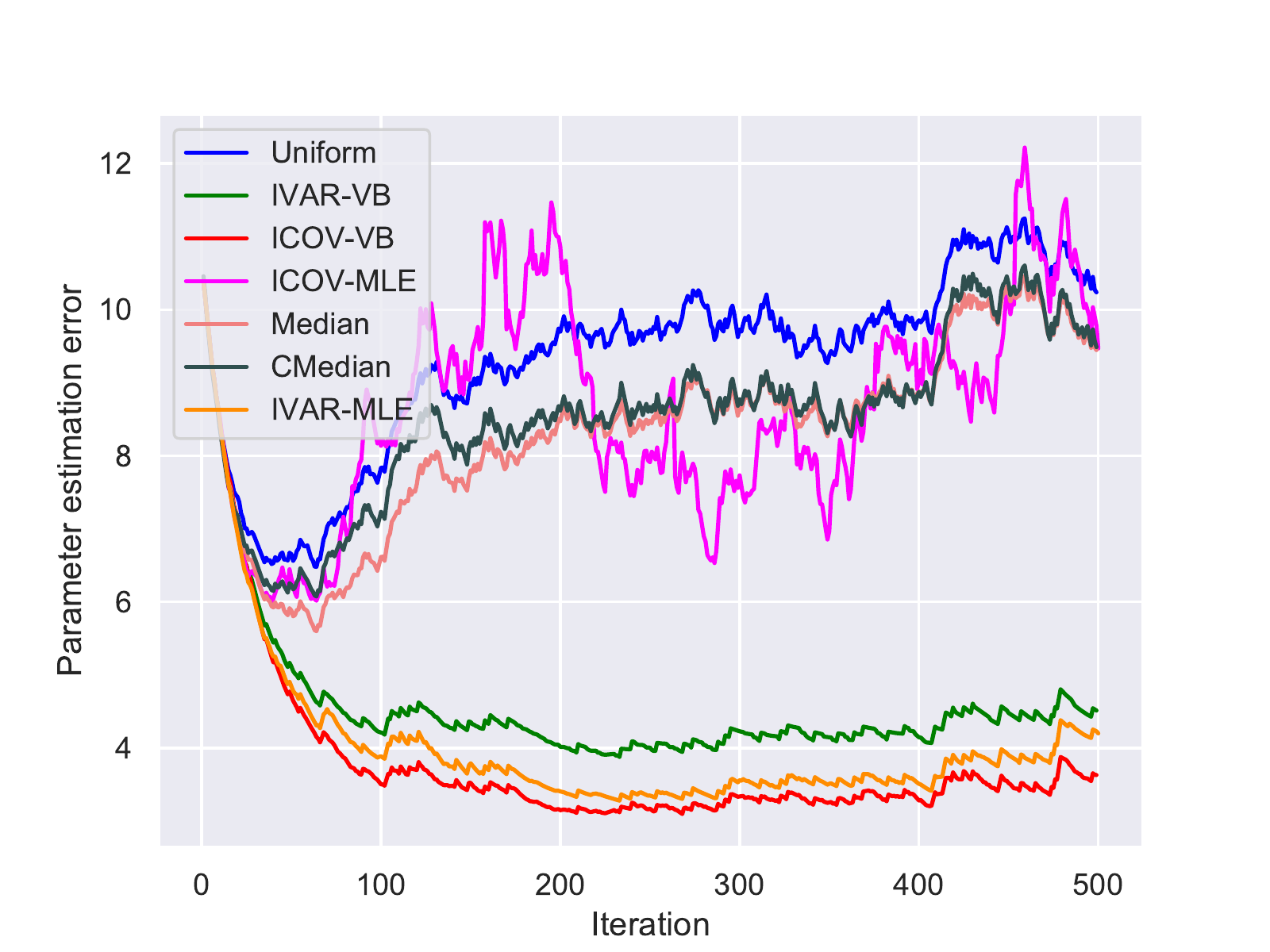} \caption{Partial:  3 random parties per round. Full batch (300  samples)}\label{fig:P_fullB}
	\end{subfigure} 
	\caption{Linear regression. ICOV and IVAR outperform  other methods when there are adversaries.}
	\label{fig:distributed_SGD_linear}
\end{figure*}


\begin{figure*}
	\centering
	\begin{subfigure}{0.25\textwidth}
		\includegraphics[width=\textwidth]{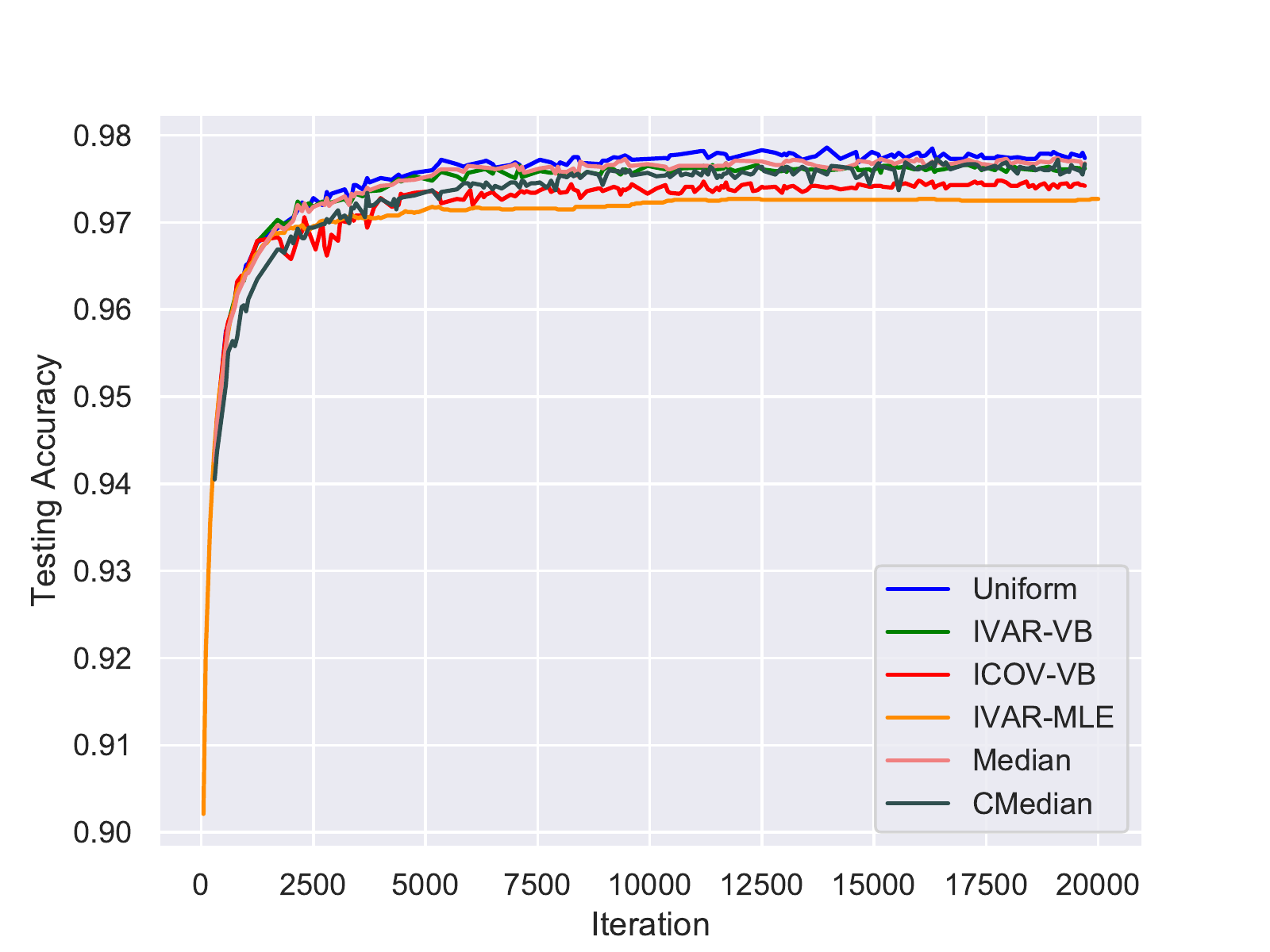} \caption{ 5 genuine parties, 0 adversaries }
	\end{subfigure} \hskip 0.2cm
	\begin{subfigure}{0.25\textwidth}
		\includegraphics[width=\textwidth]{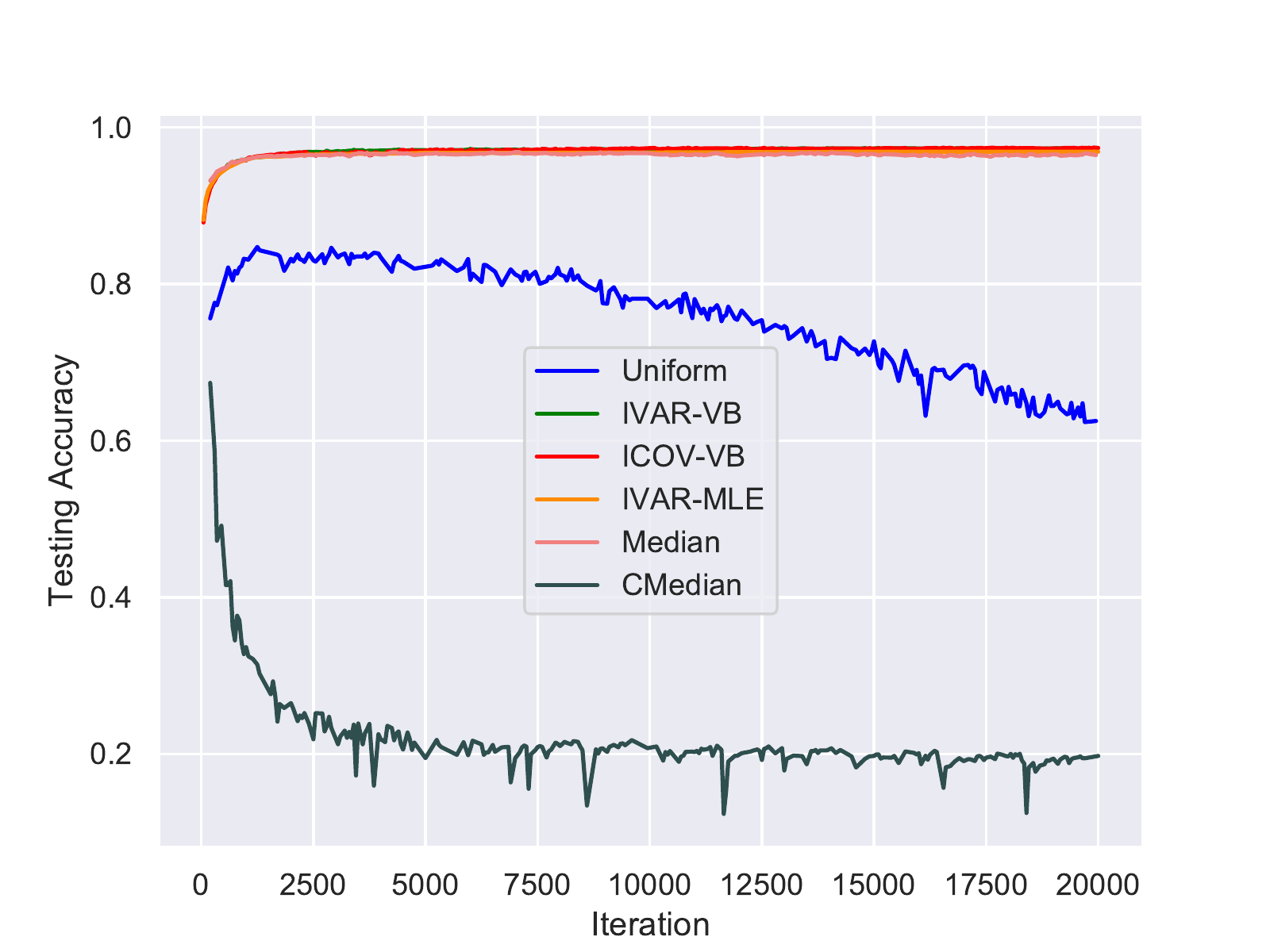} \caption{  5 genuine parties, 5 adversaries}
	\end{subfigure}  \hskip 0.2cm
	\begin{subfigure}{0.25\textwidth}
		\includegraphics[width=\textwidth]{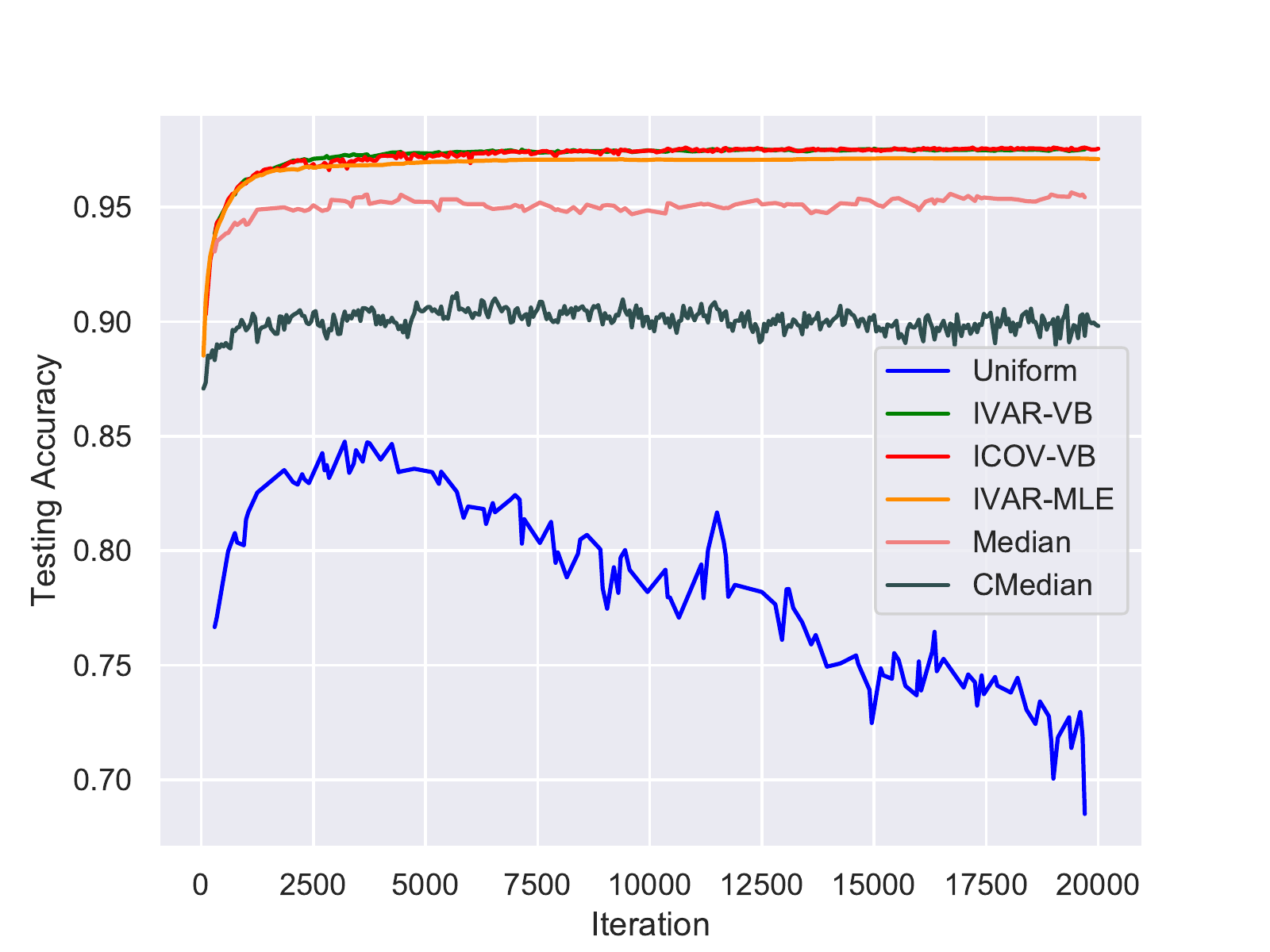} \caption{ 5 genuine parties, 10 adversar.}
	\end{subfigure} 
	\caption{Adversarial MNIST testing performance. ICOV and IVAR outperform  other methods with adversaries.}
	\label{fig:distributed_SGD}
\end{figure*}


\begin{figure*}
	\centering
	\begin{subfigure}{0.25\textwidth}
		\includegraphics[width=\textwidth]{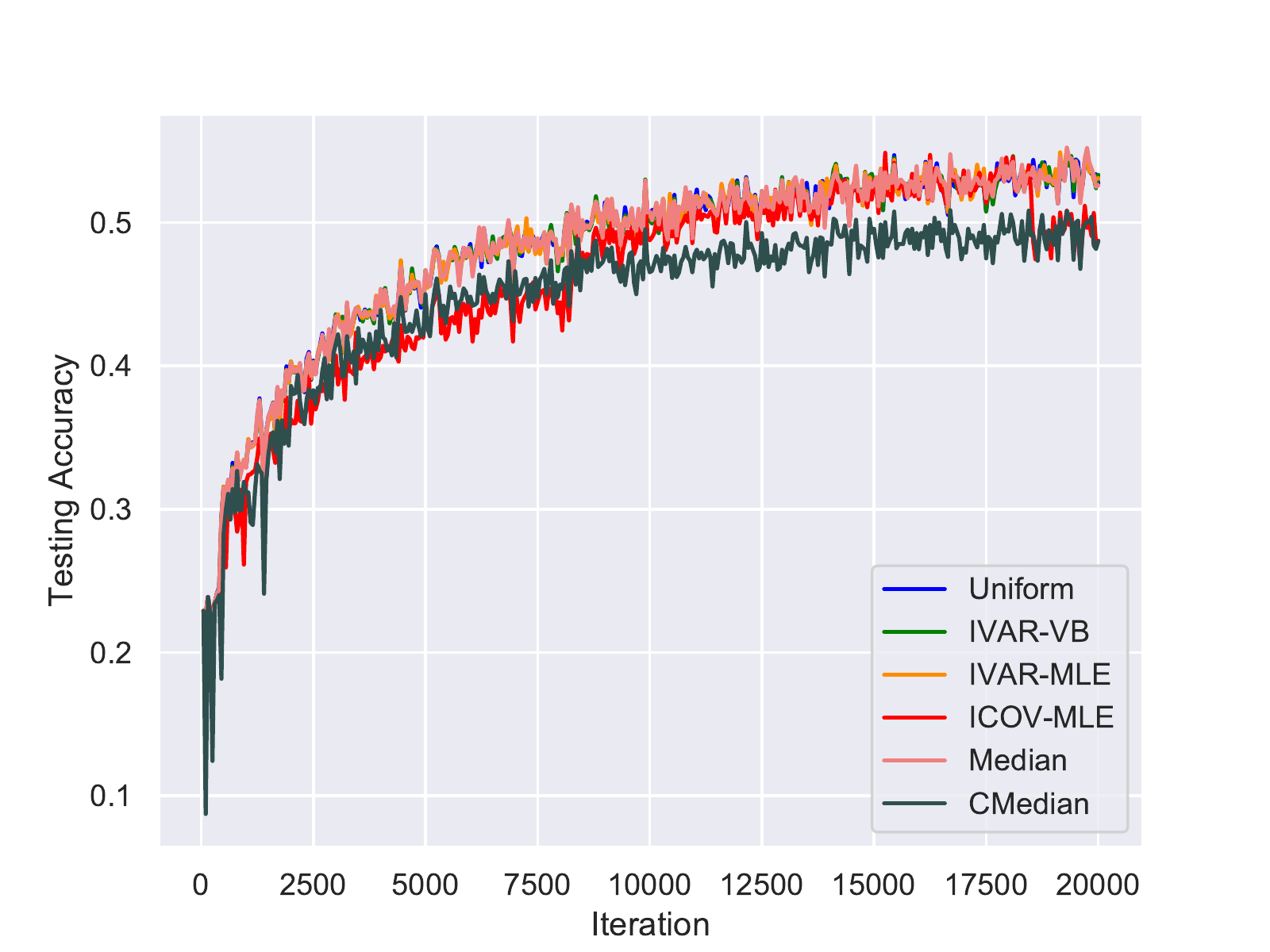} \caption{5 genuine parties, 0 adversaries  }
	\end{subfigure} \hskip 0.2cm
	\begin{subfigure}{0.25\textwidth}
		\includegraphics[width=\textwidth]{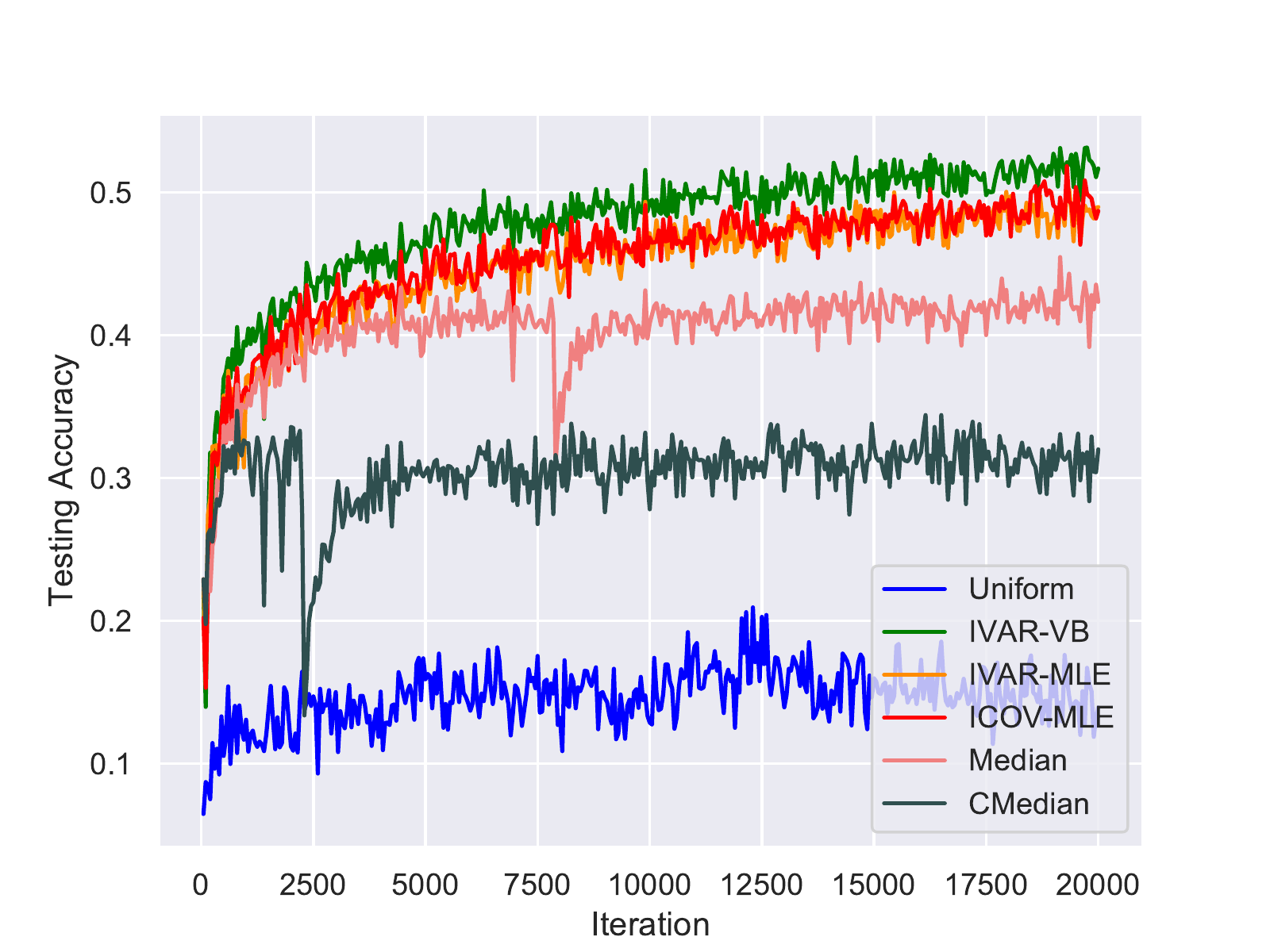} \caption{5 genuine parties, 5 adversaries }
	\end{subfigure} \hskip 0.2cm
	\begin{subfigure}{0.25\textwidth}
		\includegraphics[width=\textwidth]{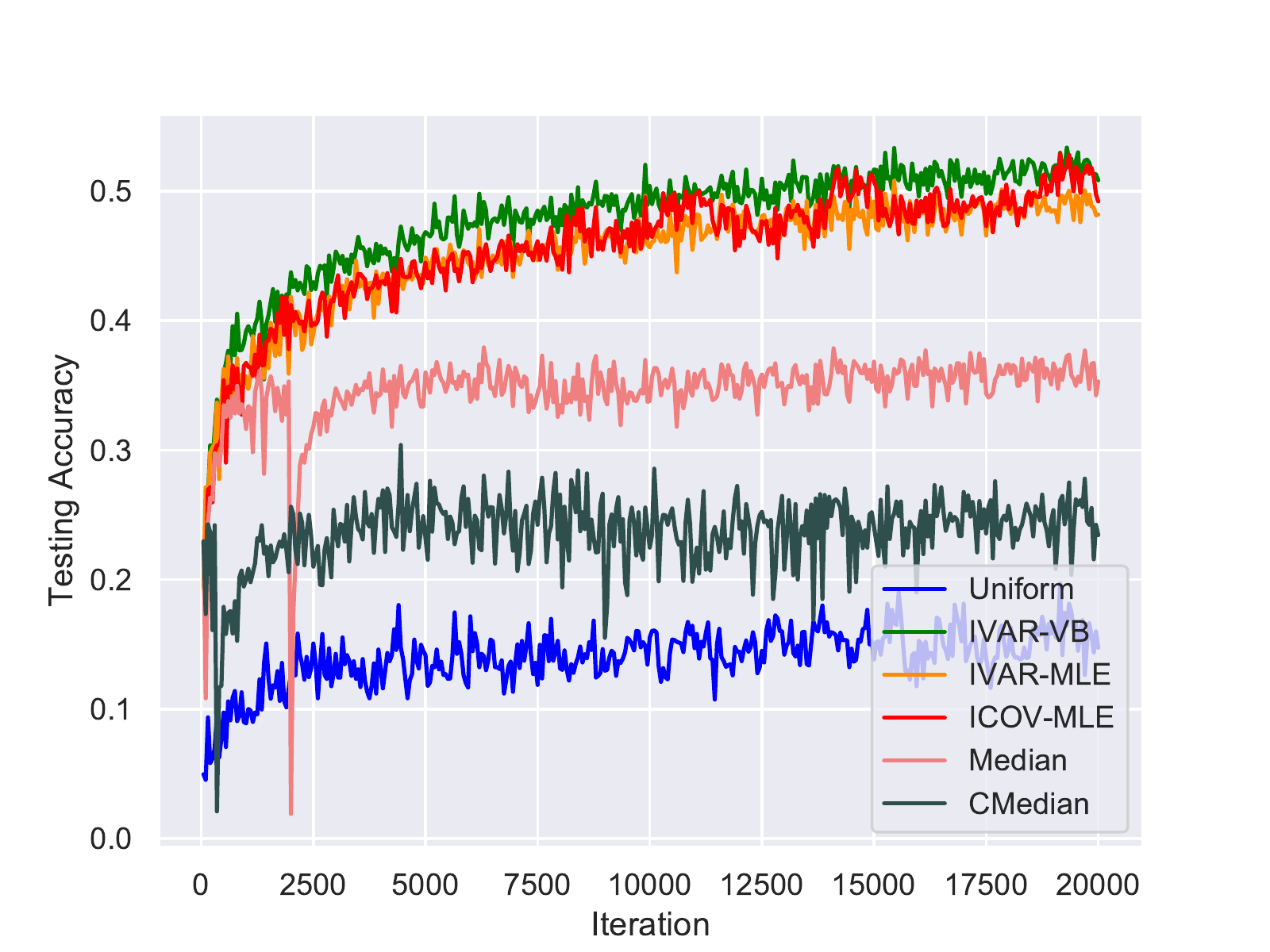} \caption{5 genuine parties, 10 adversar. }
	\end{subfigure} 
	\caption{Adversarial Shakespeare testing performance. ICOV and IVAR outperform  other methods with adversaries.}
	\label{fig:SH1}
\end{figure*}


\section{Related work}

\subsection{Robust Federated Learning}

\citet{untrusted} propose a method for federated classification and regression using a reference dataset with which to weight the parties in the federation, in a manner similar to that of \citet{refdataset2} for single-party, i.e. non-federated, training. They aggregate the parties using either the geometric median or the component-wise version thereof. Some methods such as \citet{zeno} score the contribution of each party and then accept only those up to a threshold. \citet{FL10} propose a stable variant of the geometric median algorithm for model parameter aggregation. The authors argue that parameter aggregation, as opposed to gradient aggregation, allows for more computation to occur on the devices and that assumptions on the distributions of parameters are easier to interpret. In our work we provide a mechanism to estimate the ground truth values for each party in a manner that applies to both  gradients and model parameters.

A number of works such as \citet{byz1, byz2, byz3, FL6, draco}
study the byzantine setting with assumptions on the maximum number of adversarial parties, but do not in general consider the  case of unbalanced data. \citet{byz2} propose a novel aggregation mechanism based on the distance of a party's gradients to other gradients. \citet{byz4} address the byzantine setting with non-iid data by penalizing the difference between local and global parameters, but do not consider unbalanced data. \citet{draco} offer strong guarantees but under rather strong assumptions on the collusion of the parties, running contrary to most privacy requirements, and  requiring significant redundancy with each party computing multiple gradients.
\citet{byz5} are concerned with unbalanced data in a byzantine setting where  parties  erroneously report the sample size, and so propose to truncate  weights reported by the parties to bound the impact of byzantine parties.

\subsection{Collaborative Filtering}
One of the earliest efforts in collaborative filtering was that of  \citet{dawid1979maximum} who proposed a Bayesian inference algorithm to aggregate individual worker labels and infer the ground truth  in categorical labelling. Their approach defined the two main components of a collaborative filtering algorithm: estimating the reliability of each  worker, and inferring the true label of each instance. They
applied expectation maximization and estimated the ground truth in the E-step. Then, using the
estimated ground truth, they compute the maximum likelihood estimates of the confusion matrix
in the M-step.
In continuous value labelling,  \citet{raykar2010learning} modeled each worker prediction as an independent noisy observation of the ground truth. Based on this independent noise assumption, \citet{raykar2010learning} developed a counterpart  to the  Dawid-Skene framework for the continuous domain to infer both the unknown individual variance and the ground truth.   In their M-step, the variance, which  corresponds to the confusion matrix in categorical labelling, is computed to minimize the mean square error with respect to the estimated ground truth. Their  E-step involves re-estimating the ground truth with a weighted sum of the individual predictions, where  the weights are set as  the inverses of individual variances. 
\citet{liu_variational} point to the risk of convergence to a poor-quality local optimum of the above-mentioned EM approaches and propose a variational approach for the problem.
\citet{NIPS2010_4074}   model each  worker as a multi-dimensional quantity including bias and other factors, and  group them as a function of those quantities. In federated learning, a party may also be considered to have a multidimensional set of attributes.
In collaborative filtering, workers seldom participate in all of the tasks. This  sparsity  motivates  the application of matrix factorization techniques. Federated learning also may exhibit this characteristic: if a party does not participate in all training rounds for reasons of latency, or suffers a failure, the  result would be similar to the sparsity found in collaborative filtering.
In continuous  applications    parties may exhibit correlations in their estimates.   \citet{li2019exploiting}, in the context of crowdsourced classification, showed that the incorporation of cross-worker correlations significantly improves accuracy. That work  relies on an extension of the (independent) Bayesian Classifier Combination model of \citet{iBCC} in which  worker correlation is modeled by representing  true classes
by mixtures of subtypes and motivates our inverse covariance scheme.




\section{Problem Setup and Inference Models}

Consider a global loss function
\[
F(\bm{w})=\bE_\bm{z} f(\bm{z};\bm{w})
\]
where $\bm{w}$ is the parameter of interest and $\bE$ denotes the expectation with respect
to $\bm{z}\sim \cP$ for some unknown distribution $\cP$. In a federated learning setting,
each  worker party has access to samples from $\cP$ and wish to jointly minimize $F(\bm{w})$ without revealing the local samples.
Beginning with some initial $\bm{w}=\bm{w}_0$, learning happens over single or multiple rounds where each  worker party submits a local update
to a central aggregator. The local update can be  in the form of model parameter  $ \bm{w} $ or gradient $\nabla_{\bm{w}} F(\bm{w}) $.

Each round of such updates is considered a task; we use $i=1,\ldots,I$ to index such tasks.
 Workers are indexed by $j=1,\ldots,J$. We do not assume full participation in every update round, and
use $J_i\subset\{1\ldots J\}$ to denote the set of participating workers for task $i$. Similarly, let
$I_j\subset\{1\ldots I\}$ denote the set of tasks in which worker $j$ participates. Note that the term  \textit{worker }and  \textit{party} are synonymous, as both are used in the federated learning setting.
In task $i$, each  worker $j\in J_i$ sends an update $\bm{x}_{ij}$ to the aggregator.
We make the following assumption regarding $\bm{x}_{ij}$:
\begin{assumption}\label{assumption:GD}
The local update $\bm{x}_{ij} $ follows a Gaussian distribution $\bm{x}_{ij}\sim\mathcal{N}( \bm{y}_{i},  \Sigma_j )$.
\end{assumption}
We argue that the assumption is well-founded through the following examples.

\begin{example}\label{exam_grad}
Consider a learning scheme where
each update to $\bm{w}$  computes an estimate of the global gradient
$\nabla_\bm{w} F=\bE_\bm{z} \nabla_\bm{w} f(\bm{z};\bm{w})$.
Suppose that each  worker $j$ has access to a sample $\cD_j$ of independent examples from $\cP$ and
computes $\bm{x}_{ij}=\frac{1}{|\cD_j|}\sum_{\bm{z}\in\cD_j}\nabla_\bm{w} f(\bm{z};\bm{w})$.
Let $\bm{y}_i=\bE[\nabla_\bm{w} f(\bm{z};\bm{w})]$ and
$\Sigma=\mCov[\nabla_\bm{w} f(\bm{z};\bm{w})]$. By the central limit theorem, as $|\cD_j|\to\infty$,
$\bm{x}_{ij}$ approches $\cN(\bm{y}_i,\Sigma_j)$ in distribution, with $\Sigma_j=\frac{\Sigma}{|\cD_j|}$.
\end{example}

\begin{example}
Suppose that each local update is obtained by finding the maximum likelihood estimator
for a linear model $\bm{z}_j=H_j\bm{y}_i+\bm{\epsilon}_j$ where $(H_j,\bm{z}_j)$ contains the observed local data.
Assuming that $H_j$ is fixed while $\bm{\epsilon}_j$ follows a Gaussian
distribution $\cN(0,\sigma^2\bI)$, then the least-squares solution, given by
$\bm{x}_{ij}=(H_j^\top H_j)^{-1}H_j^\top\bm{z}_j$ also follows a Gaussian $\cN(\bm{y}_i,\Sigma_j)$ where
$\Sigma_j=\sigma^2(H_j^\top H_j)^{-1}$.
\end{example}

Under Assumption~\ref{assumption:GD},
 further suppose that each local sample is \emph{independent},
the maximum likelihood estimator (MLE) for $\bm{y}_i$ is given by
\begin{align}\label{eqn_ind_mle}
\widehat{\bm{y}}_i &=\arg\max_\bm{y} \sum_j -(\bm{x}_{ij}-\bm{y})^\top\Sigma_j^{-1}(\bm{x}_{ij}-\bm{y}) \notag\\
  &=\big(\sum_j \Sigma_j^{-1} \big)^{-1}\sum_j \Sigma_j^{-1} \bm{x}_{ij}.
\end{align}
In the case of Example~\ref{exam_grad}, where $\Sigma_j=\frac{\Sigma}{|\cD_j|}$, equation~\eqref{eqn_ind_mle}
reduces to
\begin{equation}\label{eqn_fedavg}
  \widehat{\bm{y}}_i=\frac{\sum_j |\cD_j|\bm{x}_{ij}}{\sum_j |\cD_j|}.
\end{equation}
This justifies the standard averaging scheme in federated learning \citep{McMahanMRHA17}.
Note that even under the Gaussian assumption, the standard averaging scheme is the
MLE only when each  worker has independent samples.

In general, if $\bm{x}_{ij}$ and $\bm{x}_{ij'}$ are not independent, the MLE for $\bm{y}_i$ will be more
complicated. Consider the simpler case where each component in $\bm{x}_{ij}$, denoted
$x_{ij}^k$ for $k=1\ldots K$, is independent across $k$, fixing $i,j$.
On the other hand, they may be correlated among the  workers, i.e. across $j$ fixing $i,k$.
Assumption~\ref{assumption:GD} specializes to:
\begin{assumption}\label{assumption:GD1}
The local update $\bm{x}_{ij} $ follows a Gaussian distribution $\bm{x}_{ij}\sim\mathcal{N}( \bm{y}_i,  \sigma_j^2\bI )$.
Furthermore, let $\Phi$ be a $J\times J$ covariance matrix where $\Phi_{j,j}=\sigma_j^2$ and 
$\Phi_{j,j'}=\mCov(x_{ij}^k,x_{ij'}^k)$ for all $k$ and $j\neq j'$. The vector $\bm{x}_{i,:}^k=[x_{i1}^k\ldots x_{iJ}^k]^\top$
follows a Gaussian distribution $\bm{x}_{i,:}^k\sim\cN(y_i^k\one,\Phi)$.
\end{assumption}

The MLE for $\bm{y}_i$ and $\Phi$ under this setting is given by:
\begin{proposition}\label{prop_mle}
Under Assumption~\ref{assumption:GD1}, let $X_{i,\bm{j}_i}$ be the matrix whose columns
are $\bm{x}_{ij}$ for participating workers $j\in J_i$
 and $\Phi_{\bm{j}_i}$ the corresponding submatrix of $\Phi$.
The MLE for $\bm{y}_i$ (fixing $\Phi_{\bm{j}_i}$) and
$\Phi_{\bm{j}_i}$ (fixing $\bm{y}_i$) are
given, respectively, by
\begin{equation}\label{eq:icov:1}
\widehat{\bm{y}}_i=\frac{X_{i,\bm{j}_i}\Phi_{\bm{j}_i}^{-1} \one}{\one^\top\Phi_{\bm{j}_i}^{-1}\one}
\end{equation}
and
\begin{equation}\label{eqn_phi}
\widehat{\Phi}_{\bm{j}_i}=\frac{1}{K} (X_{i,\bm{j}_i}-\bm{y}_i\one^\top)^\top (X_{i,\bm{j}_i}-\bm{y}_i\one^\top).
\end{equation}
\end{proposition}
\begin{proof}
Let $\bm{x}_{i,\bm{j}_i}^k$ be the (column) vector corresponds to the $k$-th row of $X_{i,\bm{j}_i}$.
Under Assumption~\ref{assumption:GD1}, we have that $\bm{x}_{i,\bm{j}_i}^k\sim\cN(y_i^k\one,\Phi_{\bm{j}_i})$.
The log-likelihood for $\bm{x}_{i,\bm{j}_i}^k$ is given by
\begin{align*}
&\log p(\bm{x}_{i,\bm{j}_i}^k|y_i^k,\Phi_{\bm{j}_i}) \\
=&\frac{1}{2}\log|\Phi_{\bm{j}_i}^{-1}|-\frac{1}{2}
(\bm{x}_{i,\bm{j}_i}^k-y_i^k\one)^\top\Phi_{\bm{j}_i}^{-1}(\bm{x}_{i,\bm{j}_i}^k-y_i^k\one)+c
\end{align*}
for $c$ constant. The MLE can be obtained by computing
$\frac{\partial}{\partial y_i^k}\log p(\bm{x}_{i,\bm{j}_i}^k|y_i^k,\Phi_{\bm{j}_i})$
and
$\frac{\partial}{\partial (\Phi_{\bm{j}_i}^{-1})}\log p(\bm{x}_{i,\bm{j}_i}^k|y_i^k,\Phi_{\bm{j}_i})$
respectively and finding the stationary points.
\end{proof}

\begin{remark}
Note that under Assumption~\ref{assumption:GD1}, $\Phi$ is shared by
all tasks $i=1\ldots I$. Equation~\eqref{eqn_phi} can therefore be extended to  use 
the data across multiple tasks, resulting in the following update for all $j,j'$:
\[
\widehat{\Phi}_{j,j'}=\frac{1}{K|I_j\cap I_{j'}|}\sum_{i\in I_j\cap I_{j'}} (\bm{x}_{ij}-\bm{y}_i)^\top(\bm{x}_{ij'}-\bm{y}_i).
\]
\end{remark}

Let us go back to  Example~\ref{exam_grad} where each local update $\bm{x}_{ij}$ is the
average of independent examples from $\cD_j$ but for any two workers $j\neq j'$, $\cD_j$ and $\cD_{j'}$
can \emph{overlap}.
We have:
\begin{proposition}
Under Assumption~\ref{assumption:GD1},
let $\bm{x}_{ij}=\frac{1}{|\cD_j|}\sum_{\bm{g}\in\cD_j}\bm{g}$ where
$\bm{g}\sim\cN(\bm{y}_i,\sigma^2\bI)$.
Assume that for each $j$,
all $\bm{g}\in\cD_j$ are independent, but $\cD_j\cap\cD_{j'}$ may be non-empty for
any $j\neq j'$.
Then
\begin{equation}\label{eqn_overlap}
\Phi_{j,j'}=\frac{|\cD_j\cap\cD_{j'}|}{|\cD_j||\cD_{j'}|}\sigma^2.
\end{equation}
\end{proposition}
\begin{proof}
Fix a component $k$ of $\bm{g}$, we have that $g^k\sim\cN(y_i^k,\sigma^2)$.
Let $|\cD_j|=n_1+m$, $|\cD_{j'}|=n_2+m$ and $n=n_1+n_2+m$. Draw $n$ independent examples
$g^k_1\ldots g^k_n$ from $\cN(y_i^k,\sigma^2)$
such that $\cD_j^k=\{g^k_1\ldots g^k_{n_1},g^k_{n_1+n_2+1}\ldots g^k_{n_1+n_2+m}\}$ and
$\cD_{j'}^k=\{g^k_{n_1+1}\ldots g^k_{n_1+n_2},g^k_{n_1+n_2+1}\ldots g^k_{n_1+n_2+m}\}$. Note that
$m$ is the number of overlapping examples.

Let $\bm{x}=[g^k_1\ldots g^k_n]^\top$ and choose $A$ such that
$A\bm{x}=[ x_{ij}^k, x_{ij'}^k ]^\top$.
We use the fact that for a constant matrix $A$ and random vector $\bm{x}$, 
$\mCov(A\bm{x})=A\mCov(\bm{x})A^\top$. Note that $\mCov(\bm{x})=\sigma^2\bI$.
The result then follows by inspecting the entries in $A\mCov(\bm{x})A^\top$.
\end{proof}
With overlapping local samples, one can solve  the MLE of $\bm{y}_i$
using Equation~\eqref{eq:icov:1} with $\Phi$ from Equation~\eqref{eqn_overlap}.
If there is no overlap, then we  again obtain~\eqref{eqn_fedavg}.
In practice, however, it is unlikely that the aggregator has access to
the sample size as well as the sample overlap between any workers.
Our proposed approach is therefore to jointly estimate both $\bm{y}_i$ \emph{and}
the unknown $\Phi$ under Assumption~\ref{assumption:GD1}. We present in what follows two new methods for doing so. In the first we suppose that $\Phi$ is diagonal; this results in an Inverse Variance Weighting method, called IVAR. In the second we estimate the full covariance matrix, $\Phi$, in what we term Inverse Covariance Weighting, or ICOV.



\subsection{Inverse Variance Weighting}

Inverse variance weighting has been used in collaborative filtering for  aggregation without a ground truth. Inverse variance weighting has an appealing interpretation as the maximum-likelihood estimation 
under a bias-variance model,
 based on the assumption that 
 parties have independent additive prediction noise
~\cite{liu2013scoring,raykar2010learning,kara2015modeling}. As such,
the Gaussian model of Assumption~\ref{assumption:GD1} is a good approximation. 

We adapt this idea to  federated learning as follows.
Let the ground truth be $\bm{y}_i$ for each $i$.
Learning the full covariance matrix $\Phi$ can be expensive if the number of parties $J$ is large.
This justifies developing a method that uses
 a diagonal matrix with $ \Phi_{j,j'}=0 $ for $j\neq j'$.
Then,  the maximum likelihood aggregation can be computed as follows: 
\begin{proposition}
Under Assumption~\ref{assumption:GD1}, let $\Phi$ be diagonal.
The MLE for $\bm{y}_i$ (fixing $\Phi$) is given by
\begin{equation}
\label{eq:inverseVariance}
\widehat{\bm{y}}_{i} = \frac{\sum_{j\in J_i} (1/\sigma^2_{j})\bm{x}_{ij}}{\sum_{j\in J_i} 1/\sigma^2_{j}  }.
\end{equation}
For each $j$, the MLE for $\sigma_j^2$ (fixing $\bm{y}_i$) is given by
\begin{equation}\label{eqn_mle_sigma}
\widehat{\sigma}_j^2=\frac{1}{K}\|\bm{x}_{ij}-\bm{y}_i\|^2
\end{equation}
where $\|\cdot\|$ is the Euclidean norm.
\end{proposition}
\begin{proof}
The results follow  from Proposition~\ref{prop_mle}.
\end{proof}
The MLE for $\bm{y}_i$ and $\sigma_j^2$ can be jointly optimized by iterating on
Equations~\eqref{eq:inverseVariance} and~\eqref{eqn_mle_sigma}. In particular, beginning with
$\widehat{\bm{y}}_i^{(0)}$, each update is given by:
\[
\widehat{\bm{y}}_i^{(t+1)}=\frac{\sum_{j\in J_i} \big(1/\|\bm{x}_{ij}-\widehat{\bm{y}}_i^{(t)}\|^2 \big)\bm{x}_{ij}}
{\sum_{j\in J_i} \big(1/\|\bm{x}_{ij}-\widehat{\bm{y}}_i^{(t)}\|^2 \big)}.
\]
This bears a resemblance to 
 Weiszfeld's algorithm to estimate the geometric median \cite{pillutla2019robust}, where
 each update is given by:
\[
\widehat{\bm{y}}_i^{(t+1)} = \frac{\sum_{j\in J_i} \big(1/\|\bm{x}_{ij} -   \widehat{\bm{y}}_{i}^{(t)}\|\big)\bm{x}_{ij}}
{\sum_{j\in J_i} \big(1/\|\bm{x}_{ij} -   \widehat{\bm{y}}_{i}^{(t)}\|\big)  }.
\]
\begin{algorithm}[htp]
	\small
	\SetAlgoLined\DontPrintSemicolon
	\KwInput{$ \<\sigma_j, \bm{x}_{ij}\>_{ i\in \tilde{I}, j \in \mathbf{J}_i}$}
	\SetKwProg{myalg}{Algorithm}{}{}
	\For{$ t \rightarrow 1:T $}
	{
		$\bm{y}^{(t)}_{i} \gets \frac{\sum_{j\in\mathbf{j}_i} 1/\sigma^2_j\bm{x}_{ij}}{\sum_{j\in\mathbf{j}_i} 1/\sigma^2_{j}  }, \forall i\in \tilde{I}$\;
		$\sigma^2_j \gets \max\{\epsilon,(1/(K|\tilde{I}_j)|) \sum_{i\in \tilde{I}_j}\|\bm{y}^{(t)}_{i} - \bm{x}_{ij}\|^2_2\} $, $ \forall j $\;
	}
	\KwOutput{$\bm{y}^{(T)}_{i}, \<\sigma_j\>_{ j \in \mathbf{j}_i}$}
	\caption{\small Inverse-Variance Weighting Aggregator}
	\label{algo:inverse_distance}
\end{algorithm}
The algorithm for inverse variance weight aggregation, IVAR,  provided in Algorithm~\ref{algo:inverse_distance}, works as follows: upon receiving the local update  for tasks $ \tilde{I} $, the aggregator iteratively computes the ``\textit{consensus}'' $ \bm{y}^{(t)}_{i} $ for each task $ i $ using the  variance $ \sigma_j $ of each worker $ j $. Note that, as $ \sigma_j $  is assumed invariant over  tasks,  it can be  computed  as the average  variance. 

\subsection{Inverse Covariance Weighting}
\label{sec:icov}

The independence assumption in the bias-variance model 
can be violated in  federated learning scenarios when parties  use similar information and methods.
This gives rise to a collective bias within groups  of  parties.
Ideally one would like then to estimate the full covariance matrix
$\Phi$, such as using iterative updates from Proposition~\ref{prop_mle}.
The number of parameters grows with $J^2$ however and  may give poor estimations
if groups do not jointly participate in many of the tasks.
This motivates the use of a latent feature model that
allows for noise correlation across  parties
while  addressing the challenge of  sparse observations.
In particular, consider the following probabilistic model for each
local update. Without loss of generality,  let $K=1$ and omit  index $k$:
\begin{equation}\label{eqn_latent_bias}
  x_{ij}\sim
  \mathcal{N}\left(y_i+\bm{u}_i^\top \bm{v}_j,\sigma^2\right),
\end{equation}
where $\bm{u}_i\in R^D$ and $\bm{v}_j\in R^D$ are latent feature vectors 
associated with  task $i$ and worker $j$, respectively.
As such, all  observations are correlated by the unknown latent
feature vectors.
Let $X$ be the local updates over multiple tasks with
entries $X_{ij}=x_{ij}$.
Consider  maximizing the log-likelihood:
\begin{align*}
  &\log p\left(\mathbf{X}\left|\mathbf{y},\mathbf{U},\mathbf{V},\sigma^2\right.\right) \\
  =
&\sum_{(i,j): i\in I_j} \log p\left(x_{ij}\left|y_{i}, \mathbf{u}_i, \mathbf{v}_j,\sigma^2\right.\right),
\end{align*}
where  matrices 
$\mathbf{U}\coloneqq \left[\mathbf{u}_1,\dots,\mathbf{u}_I\right]^\top \in R^{I\times D}$ and 
$\mathbf{V}\coloneqq \left[\mathbf{v}_1,\dots,\mathbf{v}_J\right]^\top \in R^{J\times D}$.
In particular, we extend  inverse covariance weighting  
by a nonlinear matrix factorization technique 
based on Gaussian processes~\cite{lawrence2009non}
to jointly infer the ground truth  and the latent feature vectors.
From~\eqref{eqn_latent_bias},
observe that, by placing independent zero mean Gaussian priors 
$\mathcal{N}(\mathbf{0},\sigma_u^2\mathbf{I})$ on $\mathbf{u}_i$,
we recover the probabilistic model of Assumption~\ref{assumption:GD1}
where $\bm{x}_{i,:}\sim\cN(y_i\one,\Phi)$
with the covariance matrix:
\begin{equation*}
\Phi= \sigma_u^2 \mathbf{V}\mathbf{V}^\top + \sigma^2\bI.
\end{equation*}
Thus, the problem of covariance estimation 
has been transformed into the problem of estimating $\mathbf{V}$, $\sigma_u^2$, $\sigma^2$.
The degrees of freedom are now  
determined by the size of $\mathbf{V}$ which contains $J\times D$ values.
Since we expect $D \ll J$ in practical applications, this problem 
has significantly fewer degrees of freedom than the original problem 
of estimating the $J^2$ values of the entire covariance matrix.

Maximizing the log-likelihood involves alternating between the
optimization of $\mathbf{y}$ and $(\mathbf{V}, \sigma^2, \sigma_u^2)$. 
Specifically,  update $\mathbf{y}$ using equation~\eqref{eq:icov:1} 
and perform stochastic gradient descent on the model parameters as there 
is no closed-form solution for the latter.
The log-likelihood for round $i$ is:
\begin{align*}
  E_i(\mathbf{V},\sigma^2,\sigma_u^2)
  =
  - \log\left|\Phi_{\mathbf{j}_i}\right| 
  - \boldsymbol{\delta}_{i,\mathbf{j}_i}^\top \Phi_{\mathbf{j}_i}^{-1} \boldsymbol{\delta}_{i,\mathbf{j}_i} + \text{const.}
\end{align*}
and the gradients with respect to the parameters are:
\begin{subequations}
\begin{align}
  \nabla_{\mathbf{V}_{\mathbf{j}_i,:}} E_i(\mathbf{V},\sigma^2,\sigma_u^2)
  &=
  2\sigma_u^2
  \mathbf{G}_i \mathbf{V}_{\mathbf{j}_i,:},
  \label{eq:sgd:feature}
  \\
  \nabla_{\sigma^2} E_i(\mathbf{V},\sigma^2,\sigma_u^2)
  &=
  \mathrm{Tr}
  \left( 
    \mathbf{G}_i
  \right),
  \label{eq:sgd:param1}
  \\
  \nabla_{\sigma_u^2} E_i(\mathbf{V},\sigma^2,\sigma_u^2)
  &=
  \mathrm{Tr}
  \left(
    \mathbf{G}_i
    \mathbf{V}_{\mathbf{j}_i,:}
    \mathbf{V}_{\mathbf{j}_i,:}^\top
  \right).
  \label{eq:sgd:param2}
\end{align}
\label{eq:sgd}%
\end{subequations}
where $\bs{\delta}_{i,\bm{j}_i}=(\bm{x}_{i,\bm{j}_i}-y_i\one)$,
$\mathbf{G}_i \coloneqq 
\Phi_{\mathbf{j}_i}^{-1} 
\boldsymbol{\delta}_{i,\mathbf{j}_i} 
\boldsymbol{\delta}_{i,\mathbf{j}_i}^\top 
\Phi_{\mathbf{j}_i}^{-1} - 
\Phi_{\mathbf{j}_i}^{-1}$ and
$\mathbf{V}_{\mathbf{j}_i,:} \in R^{|J_i| \times D}$ 
is the submatrix of $\mathbf{V}$ containing 
the rows corresponding to the indices in $J_i$.
After inferring the covariance matrix, 
computing the ground truth  for new instances can be done
with Eq.~\eqref{eq:icov:1}.
One can also model the covariance matrix with non-linear kernel functions 
by replacing the inner products 
$\mathbf{v}_j^\top \mathbf{v}_{j'}$ in the covariance expression 
by a Mercer kernel function $k(\mathbf{v}_j, \mathbf{v}_{j'})$. 
The parameters in the kernel representation can be optimized by 
gradient descent on the log-likelihood function.
We focus, however, on the linear kernel 
$k(\mathbf{v}_j, \mathbf{v}_{j'}) = \mathbf{v}_j^\top \mathbf{v}_{j'} $.



\subsection{Variational Bayesian (VB) Inference}

The maximum-likelihood estimator can lead to overfitting 
when the available data is scarce, and   gradient updates
~\eqref{eq:sgd:feature}-\eqref{eq:sgd:param2} for inverse covariance
weighting are computationally
expensive.
For improved robustness and computational efficiency,
we propose
a Variational Bayesian approach
to  approximate the posterior distributions of the ground truth
under both  independent  and latent noise models. 

\paragraph{Independent Noise Model}
Under Assumption~\ref{assumption:GD1}, we place
 a prior over the ground truth $y_i$ for each $i$. Again,  assume
 $K=1$ without loss of generality. Consider the simplest prior:
 a zero-mean Gaussian $y_i\sim\cN(0,\tau^2)$
where $\tau^2$ is a hyperparameter, though this can be extended to non-zero-mean priors.
From the observed data $\bm{X}$,  estimate the full
posterior $p(\bm{y}|\bm{X})$
instead of a point estimate $\widehat{\bm{y}}$.
The variational approximate inference procedure 
approximates the posterior $p(\mathbf{y}|\mathbf{X})$
by finding the distribution $q_y$
that maximizes the (negative of the) variational free energy:
\begin{equation*}
  F\left(q_y\right)
  =
  \mathbb{E}_{q_y}
  \left[ \log
    \frac{p\left(\mathbf{X},\mathbf{y}\right)}
    {q_y(\mathbf{y})}
  \right],
\end{equation*}
where the joint probability is given by:
\begin{equation*}
  p\left(\mathbf{X},\mathbf{y}\right)
  =
  \prod_{(i,j):i\in I_j}
  p\left(x_{ij}\left| y_i\right.\right)\prod_i p\left(y_i\right).
\end{equation*}
Setting the derivative of $F$ w.r.t $q_y$ to zero 
implies that the stationary distributions are independent Gaussians:
\begin{equation*}
  q_y\left(\mathbf{y}\right)
  =
  \prod_i
  \mathcal{N}\left(y_i\left|\bar{y}_i,\lambda_i\right.\right).
\end{equation*}
where  means and covariances satisfy the following:
\begin{align}
  \lambda_i
  &=
  \left(
    \frac{1}{\tau^2}
    + \sum_{j\in J_i}\frac{1}{\sigma_j^2}
  \right)^{-1},
  \label{eq:ind-vb-truth-var}
  \\
  \bar{y}_{i}
  &=
  \lambda_i 
  \sum_{j\in J_i} \frac{x_{ij}}{\sigma_j^2} .
  \label{eq:ind-vb-truth-mean}
\end{align}
In this case, 
Eq.~\eqref{eq:ind-vb-truth-var} and~\eqref{eq:ind-vb-truth-mean}
 provide the exact posterior for the ground truth $\bm{y}$
given $\bm{X}$. Updating the hyperparameters by
minimizing the variational free energy results in:
\begin{align}
  \tau^2
  &= \frac{1}{I}\sum_i \lambda_i + \bar{y}_{i}^2,
  \label{eq:ind-vb-truth-param}
  \\
  \sigma_j^2
  &=
  \frac{1}{|I_j|}
  \sum_{i\in I_j} 
  \left(
    \lambda_i + \left(x_{ij} - \bar{y}_{i}\right)^2 
  \right).
  \label{eq:ind-vb-error-param}
\end{align}
In summary, the proposed approach performs block coordinate descent
by applying repeatedly eq.~\eqref{eq:ind-vb-truth-var} to~\eqref{eq:ind-vb-error-param} 
and aggregates using the posterior mean $\bar{y}_i$.

\paragraph{Latent Noise Model}

One of the key steps in
the MLE approach to Inverse Covariance Weighting
is the marginalization of $\mathbf{U}$ 
conditioned on $(\mathbf{V},\sigma^2,\sigma_u^2)$. 
This  can be interpreted as 
 Bayesian averaging over $\mathbf{U}$. 
However,  
full Bayesian averaging over both $\mathbf{U}$ and $\mathbf{V}$ is challenging,
motivating the
Variational Bayes approach.
First,  place zero mean Gaussian priors
on the latent variables:
\begin{align*}
 p\left(y_i,\sigma_y^2\right) 
 = \mathcal{N}\left(y_i\left|0,\sigma_y^2\right.\right),
 \\
 p\left(\mathbf{u}_i,\sigma_u^2\right) 
  = \mathcal{N}\left(\mathbf{u}_i\left|\mathbf{0},\sigma_u^2\mathbf{I}\right.\right),
  \\
  p\left(\mathbf{v}_j,\sigma_v^2\right) 
  = \mathcal{N}\left(\mathbf{v}_j\left|\mathbf{0},\sigma_v^2\mathbf{I}\right.\right),
\end{align*}
where $\sigma_y^2$, $\sigma_u^2$, $\sigma_v^2$ are hyperparameters.
For notational brevity, we  omit the dependence of the distributions 
on the hyperparameters $\sigma^2$, $\sigma_y^2$, $\sigma_u^2$, $\sigma_v^2$.
The variational  inference procedure finds distributions 
that maximize the (negative of the) variational free energy of the model from~\eqref{eqn_latent_bias}, assuming
a factored distribution $q(\bm{y},\bm{U},\bm{V})=q_y(\bm{y})q_u(\bm{U})q_v(\bm{V})$:
\begin{equation*}
  F\left(q_y,q_u,q_v\right)
  =
  \mathbb{E}_{q_y,q_u,q_v}
  \left[\log
    \frac{ p\left(\mathbf{X},\mathbf{y},\mathbf{U},\mathbf{V}\right)}
    {q_y(\bm{y})q_u(\mathbf{U})q_v(\mathbf{V})}
  \right],
\end{equation*}
where the joint probability is:
\begin{align*}
  p\left(\mathbf{X},\mathbf{y},\mathbf{U},\mathbf{V}\right)
  &=
  \prod_{(i,j):i\in I_j}
  p\left(x_{ij}\left| y_i,\mathbf{u}_i,\mathbf{v}_j\right.\right)
  \\
  &\qquad\;
  \times
  \prod_i p\left(y_i\right)
  \prod_i p\left(\mathbf{u}_i\right)
  \prod_j p\left(\mathbf{v}_j\right).
\end{align*}
Then, solve for $q_y$, $q_u$ and $q_v$ by performing 
block coordinate descent on $F$. The resulting posterior distributions
are  Gaussians where $q_y(\bm{y})=\prod_i \cN(y_i|\bar{y}_i,\lambda_i)$,
$q_u(\bm{U})=\prod_i\cN(\bm{u}_i|\bar{\bm{u}}_i,\Phi_i)$, and
$q_v(\bm{V})=\prod_j\cN(\bm{v}_j|\bar{\bm{v}}_j,\Psi_j)$.
The means and covariances are given by:
\begin{align}
  \lambda_i
  &=
  \left(
    \frac{1}{\sigma_y^2}
    + \sum_{j\in J_i}\frac{1}{\sigma^2}
  \right)^{-1},
  \label{eq:latent-vb-truth-var}
  \\
  \bar{y}_{i}
  &=
  \lambda_i 
  \sum_{j\in J_i} \frac{1}{\sigma^2} \left(x_{ij} - \bar{\mathbf{u}}_i^\top \bar{\mathbf{v}}_j\right),
  \label{eq:latent-vb-truth-mean}
  \\
  \boldsymbol\Phi_i
  &=
  \left(
    \frac{1}{\sigma_u^2}\mathbf{I}
    + \sum_{j\in J_i}\frac{1}{\sigma^2}
    \left(\boldsymbol\Psi_j + \bar{\mathbf{v}}_j\bar{\mathbf{v}}_j^\top\right)
  \right)^{-1},
  \label{eq:latent-vb-party-var}
  \\
  \bar{\mathbf{u}}_i 
  &= 
  \boldsymbol\Phi_i 
  \sum_{j\in J_i} \frac{1}{\sigma^2}\left(x_{ij} - \bar{y}_{i}\right) \bar{\mathbf{v}}_j,
  \label{eq:latent-vb-party-mean}
  \\
  \boldsymbol\Psi_j
  &=\left(
    \frac{1}{\sigma_v^2} \mathbf{I} 
  + \sum_{i \in I_j} \frac{1}{\sigma^2}
  \left(\boldsymbol\Phi_i + \bar{\mathbf{u}}_i\bar{\mathbf{u}}_i^\top\right)
  \right)^{-1},
  \label{eq:latent-vb-task-var}
  \\
  \bar{\mathbf{v}}_j
  &= 
  \boldsymbol{\Psi}_j
  \sum_{i\in I_j}\frac{1}{\sigma^2}\left(x_{ij} - \bar{y}_{i}\right)\bar{\mathbf{u}}_i.
  \label{eq:latent-vb-task-mean}
\end{align}
The hyperparameter updates are given by:
\begin{align}
  \sigma_y^2
  &=
  \frac{1}{I}
  \left(\sum_i \left(\lambda_i + \bar{y}_i^2\right) \right),
  \label{eq:latent-vb-gt-param}
  \\
  \sigma_u^2
  &=
  \frac{1}{DI}
  \left(\sum_i \mathrm{Tr}\left(\boldsymbol\Phi_i + \bar{\mathbf{u}}_i \bar{\mathbf{u}}_i^\top \right)\right),
  \label{eq:latent-vb-party-param}
  \\
  \sigma_v^2
  &=
  \frac{1}{DJ}
  \left(\sum_j \mathrm{Tr}\left(\boldsymbol\Psi_j + \bar{\mathbf{v}}_j \bar{\mathbf{v}}_j^\top \right)\right),
  \label{eq:latent-vb-task-param}
  \\
  \sigma^2
  &=
  \frac{1}{\sum_{j} |I_j|}
  \sum_{(i,j):i\in I_j} \left[
    \lambda_i + \left(x_{ij} - \bar{y}_{i}\right)^2 
    - 2 \left(x_{ij} - \bar{y}_{i}\right) \bar{\mathbf{u}}_i^\top \bar{\mathbf{v}}_j \right.
  \nonumber
  \\
  & \qquad\qquad\quad\; 
    + \mathrm{Tr}\left(
    \left(\boldsymbol\Psi_i + \bar{\mathbf{u}}_i\bar{\mathbf{u}}_i^\top\right)
    \left(\boldsymbol\Phi_j + \bar{\mathbf{v}}_j\bar{\mathbf{v}}_j^\top\right)
  \right)
\Big].
  \label{eq:latent-vb-error-param2}
\end{align}
In summary, the algorithm applies 
equations~\eqref{eq:latent-vb-truth-var} to~\eqref{eq:latent-vb-error-param2} 
repeatedly until convergence. 
\begin{algorithm}[htb]
	\small
	\SetAlgoLined\DontPrintSemicolon
	\KwInput{$ \<\mathbf{v}_j, \sigma_j, \bm{x}_{ij}\>_{ j \in \mathbf{j}_i}$}
	\SetKwProg{myalg}{Algorithm}{}{}
	\For{$ t \rightarrow 1:T $}
	{
		$ \boldsymbol\Sigma_{\mathbf{j}_i}= \sigma_u^2 \mathbf{V}_{\mathbf{j}_i}\mathbf{V}_{\mathbf{j}_i}^\top + \textbf{diag}(\sigma^2_{\mathbf{j}_i}) $\;
		$ y^{(t)}_i= \frac{\mathbf{1}^\top {\boldsymbol\Sigma}_{\mathbf{j}_i}^{-1} x_{i, \mathbf{j}_i}}
		{\mathbf{1}^\top {\boldsymbol\Sigma}_{\mathbf{j}_i}^{-1}\mathbf{1}} $\;
		Update using ~\eqref{eq:latent-vb-truth-var}-\eqref{eq:latent-vb-error-param2}.
	}
	\KwOutput{$\bm{y}^{(T)}_{i},  \<\mathbf{v}_j, \sigma_j, \bm{x}_{ij}\>_{ j \in \mathbf{j}_i}$}
	\caption{\small Inverse Covariance Weighting Aggregator }
	\label{algo:rl}
\end{algorithm}


%
%
%

\begin{table}[ht]
		\begin{tabular}{l|l|l|l|}
		\cline{2-4}
		& Synthetic & MNIST  & Shakespeare \\ \hline
		\multicolumn{1}{|l|}{Uniform avg.} & 10.17    & 0.4926 & 0.16      \\ \hline
		\multicolumn{1}{|l|}{Geom. media} & 8.13   & 0.5233 & 0.41      \\ \hline
		\multicolumn{1}{|l|}{Coord. median} & 6.131     & 0.7987 & 0.29      \\ \hline
		\multicolumn{1}{|l|}{\ivar-VB} &4.62     & 0.8943 & \textbf{0.56 }     \\ \hline
		\multicolumn{1}{|l|}{\ivar-MLE} & 4.66     & \textbf{0.9043} & 0.50     \\ \hline
		\multicolumn{1}{|l|}{\icov-VB}& \textbf{2.89}    & 0.8932 & 0.52      \\ \hline
		\multicolumn{1}{|l|}{\icov-MLE } & 8.75     & 0.5253 & N.A      \\ \hline
	\end{tabular}
		\caption{Performance of the federated learning aggregation algorithms, uniform averaging, geometric median, and coordinate-wise median, against proposed IVAR and ICOV,  MLE and  VB versions. In the Synthetic linear regression example, with full participation of 5 genuine parties and full batch, prediction error is shown, hence lower  is better.  On the one-round  MNIST task  and 5 genuine parties and 5 adversaries, prediction accuracy is shown, so higher is better. In the multi-round stochastic gradient aggregation task using the Shakespeare dataset, with 5 genuine parties and 5 adversaries accuracy is provided so again higher is better.}
		\label{tab:summary}
\end{table}


\section{Experiments}

We present experimental results with a synthetic dataset and two real datasets: MNIST and  Shakespeare \cite{McMahanMRHA17}. 
We compare 	 (1)  Uniform averaging
	(2)  Geometric median which uses the smoothed Weiszfeld algorithm of~\citet{pillutla2019robust}
		(3) Coordinate-wise median which uses the coordinate-wise median as in \citet{YinCRB18}
	(4) our proposed IVAR, using the MLE formulation and using the VB
	(5) our proposed ICOV,  again using the MLE formulation and using  VB, which computes a low-rank estimation of the covariance matrix.

\paragraph{Synthetic dataset experiment}

We design a synthetic linear regression experiment  to create an environment where each party in the federation has a different noise level, and   the local data of each party is overlapping. The experimental setup is provided in the Supplementary Materials.
Figure~\ref{fig:distributed_SGD_linear} shows the algorithm performance  for various levels of participation and    batch size. 
ICOV  performs better than IVAR, and both ICOV and IVAR outperform the other baselines.

\paragraph{MNIST}

In this adversarial MNIST  classification task,   a Gaussian adversary  submits a random vector with components generated from a standard normal, $ \mathcal{N}(0, 1) $. We first study  one-round parameter estimation using using logistic regression, as in \citet{YinCRB18} with 5 genuine parties and  $R\in[0,10]$ adversaries.  
Bayesian inference aggregation IVAR and ICOV outperform the other algorithms including    robust estimators coordinate-wise median and geometric median when the number of adversaries increases. Results show the training convergence of  IVAR, ICOV and the   geometric median. 
  IVAR and geometric median convergence are fast with less than 5 iterations. ICOV  convergence is slower, but with a large   number of adversaries, ICOV converges to a better solution than IVAR. The geometric median  is less robust than the component-wise median in  one-round estimation. Details and results for this setting can be found in the Supplementary Materials.

Next, we solve adversarial MNIST using  distributed stochastic gradient descent (SGD)  with the architecture of~\citet{baruch2019little}.
 Figure~\ref{fig:distributed_SGD} shows that when there is no adversary, uniform  aggregation  is ideal. However, with adversaries, both uniform averaging and coordinate-wise median perform poorly. When adversaries account for more than half of the parties,  the Bayesian methods IVAR and ICOV are superior.

\paragraph{Shakespeare}
Lastly, we consider an NLP task using the Shakespeare dataset.  Results, shown in Figure \ref{fig:SH1}, illustrate the  case where an adversary  submits a random  vector generated from a normal distribution in place of its true parameter vector.  The different setting where the adversary performs a random local update can be found in the Supplementary Materials. Across the board IVAR-VB is shown to be superior to the other methods.

The results are summarized in Table \ref{tab:summary}, and further details are provided in the Supplementary Materials. Note that the synthetic dataset is measured in terms of error, so that a lower number is better, while the MNIST and Shakespeare tasks report classification accuracy,  so higher is better. Across the board, the proposed methods are far superior to both standard averaging and robust aggregation algorithms. It can be noted that the choice of which variant of the proposed methods is superior depends upon the task. Overall, the MLE version of ICOV tends to be computationally challenging, but the VB version of ICOV is very competitive. The IVAR method using both MLE and VB is an ideal choice when overlap is not extensive, as is the case in the MNIST and Shakespeare tasks.

\section{Discussion} We  proposed new methods for federated learning aggregation on heterogeneous data. Given that data heterogeneity in federated learning is similar to estimating the ground truth in collaborative filtering,  we  adapt techniques to estimate the uncertainty of the party updates so as to appropriately weight their contribution to the federation. The techniques involve both MLE and Variational Bayes estimators and in the simplest setting reduce to the standard average aggregation step. In more general cases, including data overlap, they provide new techniques, which enjoy superiority in  the synthetic and real world datasets examined. We expect that these methods  will help  make federated learning applicable to a wider variety of real world problems.

\bibliography{aaai21}

\end{document}